\documentclass{article}

 \usepackage[nonatbib,main,final]{neurips_2026}

\usepackage[utf8]{inputenc} 
\usepackage[T1]{fontenc}    
\usepackage{hyperref}       
\usepackage{url}            
\usepackage{booktabs}       
\usepackage{amsfonts}       
\usepackage{nicefrac}       
\usepackage{microtype}      
\usepackage{xcolor}         

\usepackage{amsmath}
\usepackage{multirow}
\usepackage{graphicx}
\usepackage{algorithm}
\usepackage{algpseudocode}
\usepackage{multicol}
\usepackage{multirow}
\usepackage{graphicx}
\usepackage{enumitem}
\usepackage[numbers]{natbib}
\usepackage{amsthm}      
\usepackage{amsmath, amssymb}

\newtheorem{theorem}{Theorem}[section]    
         
\newtheorem{proposition}[theorem]{Proposition}

\usepackage{ulem}
\usepackage{booktabs}
\usepackage{wrapfig}

\newcommand{\boldres}[1]{{\textbf{\textcolor{red}{#1}}}}
\newcommand{\secondres}[1]{{\underline{\textcolor{blue}{#1}}}}

\makeatletter
\renewcommand{\@notice}{}
\makeatother

\title{Learning Spatio-Temporal Foundation Models\\ from Pure Synthetic Data}

\author{
    \textbf{Yutong Feng\textsuperscript{1}\textsuperscript{2}\thanks{Work done at Hong Kong University of Science and Technology (Guangzhou).}, Shiyuan Piao\textsuperscript{3}, Yutong Xia\textsuperscript{2}, Xu Liu\textsuperscript{2}, Wenqi Fan\textsuperscript{4},} \\
    \textbf{Fugee Tsung\textsuperscript{3}, See-Kiong Ng\textsuperscript{2}, Yuxuan Liang\textsuperscript{1}\thanks{Y. Liang is the corresponding author. Email: yuxliang@outlook.com}} \\
    \textsuperscript{1}Hong Kong University of Science and Technology (Guangzhou) \\
    \textsuperscript{2}National University of Singapore \\
    \textsuperscript{3}Hong Kong University of Science and Technology \\
    \textsuperscript{4}Hong Kong Polytechnic University \\
    \texttt{ytfeng.caspian@163.com; shiyuan.piao@connect.ust.hk}\\
    \texttt{yuxliang@outlook.com}\\
}

\begin{document}

\maketitle

\begin{abstract}
  Spatio-Temporal Foundation Models (STFMs) aim to learn generalizable representations of complex dynamical systems across space and time. However, existing approaches suffer from distributional bias in real-world pre-training data, structural bottlenecks of autoregressive or diffusion-based paradigms, and objectives that overemphasize point-wise reconstruction in noisy observation space.
We propose \textbf{NeoST}, the first spatio-temporal foundation model pre-trained solely on procedurally generated synthetic systems. NeoST introduces a scalable synthetic pre-training corpus to mitigate real-world bias, a latent-space reasoning architecture that generates and iteratively refines multiple future trajectories without sequential error accumulation, and latent-space objectives that emphasize structural dynamics and enable inference-time correction under distribution shifts.
Extensive experiments across diverse real-world benchmarks show that NeoST consistently outperforms existing STFMs in diverse real-world spatio-temporal systems, achieves superior long-horizon stability and inference efficiency.
\end{abstract}

\section{Introduction}

Spatio-Temporal Foundation Models (STFMs) \cite{liang2025foundation,fang2026unraveling,goodge2025spatio} are large-scale pre-trained models that learn generalizable representations of spatio–temporal dependencies from diverse and heterogeneous spatio-temporal data, enabling a single unified backbone to support multiple downstream tasks across domains through zero-shot or lightweight adaptation. 
Unlike conventional task-specific spatio-temporal models \cite{jin2023spatio} that are trained independently for each dataset and objective, STFMs aim to distill universal structural priors of space and time into a scalable parameterization, shifting the paradigm from one-model-per-task engineering to foundation-level representation learning. 
STFMs unify fragmented modeling paradigms, enabling label-efficient cross-domain transfer and robust generalization under distribution shifts. Furthermore, by distilling the structural evolution of complex dynamic systems \cite{feng2025flownet}, STFMs establish a critical foundation for data-driven World Models \cite{qu2026representation,zhang2026hierarchical}.

Despite their promise, the development of STFMs faces fundamental challenges across the entire modeling pipeline, beginning with the quality and diversity of available training data \cite{fang2026unraveling}. 
Models pre-trained on specific real-world collections often overfit to the idiosyncratic topology and temporal rhythms of those specific sensor deployments, hindering their generalizability to unseen scenarios.
This imbalance induces strong distributional bias during pre-training, leading models to overfit domain-specific patterns and underperform when transferred to less-represented systems. Additionally, real-world spatio-temporal data are often plagued by sensor noise, irregular sampling, and extensive missing values \cite{wang2023observed}. Such imperfections not only degrade representation quality but also entangle intrinsic system dynamics with measurement artifacts, making it difficult for STFMs to distill clean and universal structural priors of space and time.

Beyond data limitations, current STFMs are largely built upon either autoregressive \cite{liang2024foundation,benidis2022deep,lim2021time} or diffusion-based paradigms \cite{10.1145/3783986,yang2024survey,wen2023diffstg,hu2024towards}, each of which introduces structural bottlenecks. Autoregressive models generate future states sequentially, conditioning each prediction on previously generated outputs. While conceptually straightforward, this paradigm is prone to error accumulation: early prediction mistakes propagate and amplify over long horizons, hindering stable long-term forecasting and limiting the model’s capacity for global reasoning and revision. Diffusion-based models alleviate some issues of sequential dependency by modeling complex distributions through iterative denoising, but they typically require dozens to hundreds of sampling steps during inference. Such computational overhead is prohibitive in time-sensitive scenarios, such as disaster response or real-time urban management, where rapid and reliable predictions are essential.

Challenges further arise in training objectives and inference paradigms \cite{fang2026unraveling}. Spatio-temporal signals often exhibit sparse or weak semantics in the raw time domain, where point-wise losses \cite{jadon2024comprehensive} encourage models to focus on local fluctuations while overlooking global structures such as trends, periodicities, and latent dynamical mechanisms \cite{cuturi2017soft,le2019shape}. Optimizing directly in the observation space may therefore bias the model toward short-term fidelity rather than structural understanding. In addition, most existing STFMs lack effective test-time adaptation capabilities. When confronted with distribution shifts, such as stemming from seasonal changes, policy interventions, or sensor redeployments, these models typically treat the observed input sequence as fixed context, without explicitly modeling or correcting for the shift. The inability to leverage incoming observations to recalibrate predictions at inference time significantly constrains robustness and generalization in open-world settings.

\begin{figure}
    \centering
    \vspace{-20pt}
    \includegraphics[width=0.95\linewidth]{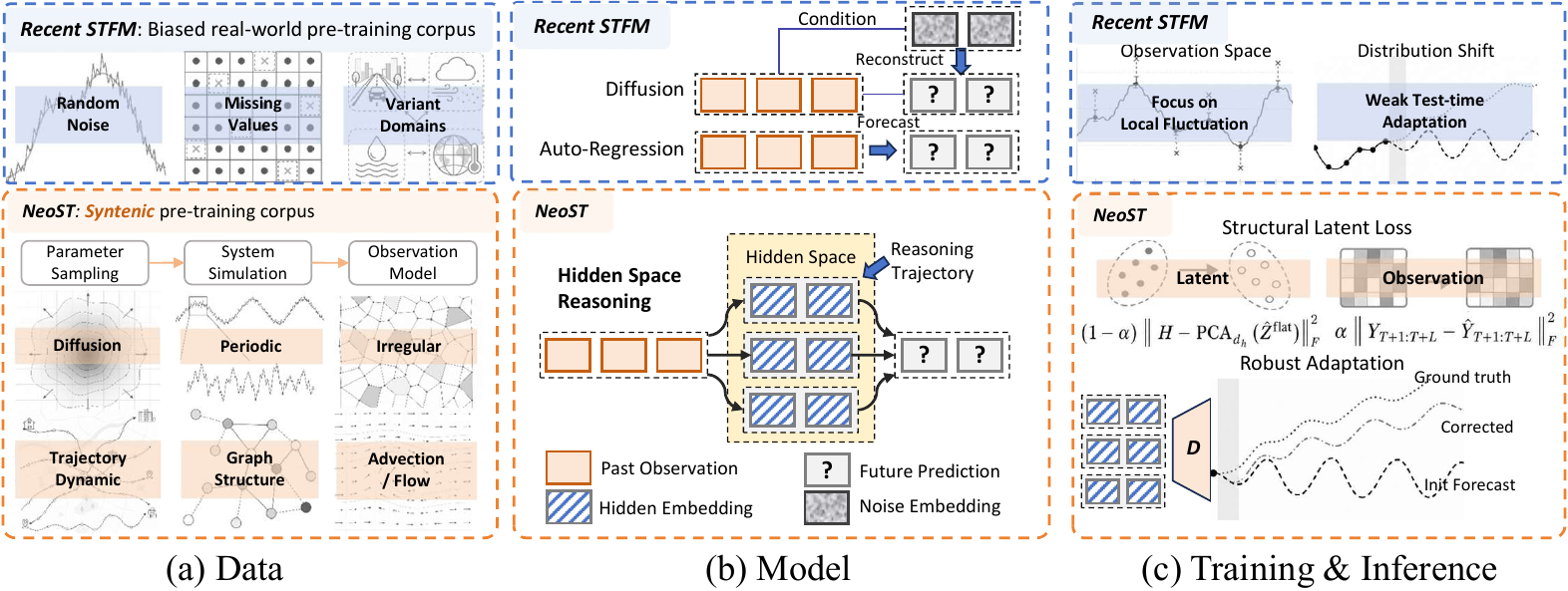}
    \caption{Comparison of Paradigms: NeoST vs. Recent STFM.}
    \vspace{-13pt}
    \label{fig:intro}
\end{figure}

To tackle these challenges, we propose \textbf{NeoST}, the first spatio-temporal foundation model \textbf{pre-trained solely on synthetic spatio-temporal data}. Our main contributions can be summarized as the following three aspects:
\begin{itemize}[leftmargin=*]
    \item On the \textbf{data} side, we depart from conventional reliance on large but biased real-world datasets and instead construct a synthetic pre-training corpus with virtually unlimited scale and controllable diversity. By procedurally generating spatio-temporal systems that span a broad range of dynamics, noise levels, and structural patterns, we mitigate distributional bias and avoid the detrimental effects of missing values and measurement noise in low-quality pre-training data. 

    \item On the \textbf{architecture} side, we propose a latent-space reasoning paradigm for STFMs, replacing the dominant autoregressive and diffusion-based formulations. Our model generates multiple candidate future trajectories directly in latent space and performs iterative refinement through patch-level and trajectory-level correction mechanisms. 
    By decoupling reasoning from step-wise generation, NeoST is capable of deeper deliberation over alternative futures under a fixed computational budget.
    \item On the \textbf{training \& inference} side, NeoST compute losses in latent space rather than directly in the raw time domain. This approach encourages the model to learn and reason about the underlying evolutionary dynamics and structural patterns of spatio-temporal systems, rather than merely fitting superficial value-level details. 
    Moreover, the latent reasoning process naturally enables inference-time correction based on observed inputs, enhancing robustness under distribution shifts.
   
\end{itemize}


\section{Related Works}

The success of large-scale pre-training in NLP has inspired analogous efforts in time series and spatio-temporal data \cite{jin2023large,ni2026streasoner}. Time Series Foundation Models (TSFMs) \cite{liang2024foundation} such as TimesFM \cite{Das2024TimesFM}, Chronos \cite{ansari2024chronos}, Moirai \cite{woo2024unified,liu2024moirai,liu2025moirai}, UniTime \cite{liu2024unitime}, MOMENT \cite{goswami2024moment}, and Time-MOE \cite{shi2024time} consolidate multiple forecasting tasks within a single pre-trained backbone, demonstrating the viability of foundation-scale temporal representation learning. However, these models ignore the rich spatial structure present in real-world systems.
Emerging STFMs extend this paradigm to spatially resolved domains \cite{liang2025foundation,goodge2025spatio}. OpenCity \cite{li2024opencity} integrates Transformer-GNN hybrids pre-trained on heterogeneous traffic datasets to enable zero-shot urban forecasting. UniST \cite{yuan2024unist} introduces knowledge-guided prompts for unified urban spatio-temporal prediction across few-shot and zero-shot settings. 
UrbanGPT \cite{li2024urbangpt} couples spatio-temporal encoders with LLM prompting for instruction-based spatial reasoning under limited supervision. 
In climate science \cite{liang2023airformer,ijcai2024p0803}, models such as Pangu-Weather \cite{bi2023accurate} and ClimaX \cite{nguyen2023climax} apply large Transformer architectures to global reanalysis data, achieving spatio-temporal generality across physical domains.
Though promising, existing STFMs share common limitations: they rely on real-world data that is domain-biased and noisy, adopt either error-prone autoregressive decoding or computationally intensive diffusion sampling, and optimize directly in the observation space without structural abstraction. 
\textit{An expanded version of the related works can be found in the Appendix \ref{appendix:related works}.}

\section{Synthetic Spatio-Temporal Data}
\label{sec:data}

We construct our synthetic data engine by explicitly parameterizing a family of stochastic graph dynamical systems defined on a graph 
$\mathcal{G} = (\mathcal{V}, \mathcal{E})$ 
with $|\mathcal{V}| = N$. 
Let $\mathbf{A} \in \mathbb{R}^{N \times N}$ denote the adjacency matrix and 
$\mathbf{L}$ the corresponding normalized graph Laplacian. 
At each time step $t$, the latent dynamical state is represented as 
$\mathbf{H}_t \in \mathbb{R}^{N \times d_h}$, 
where $d_h$ denotes the latent channel dimension. 
The observable signal $\mathbf{x}_t \in \mathbb{R}^{N}$ 
is obtained through a random linear projection of $\mathbf{H}_t$, while the latent state itself remains fully accessible during training.
Further analysis is provided in Appendix \ref{appendix:theory}.

\begin{figure}[!h]
    \centering
    \includegraphics[width=\linewidth]{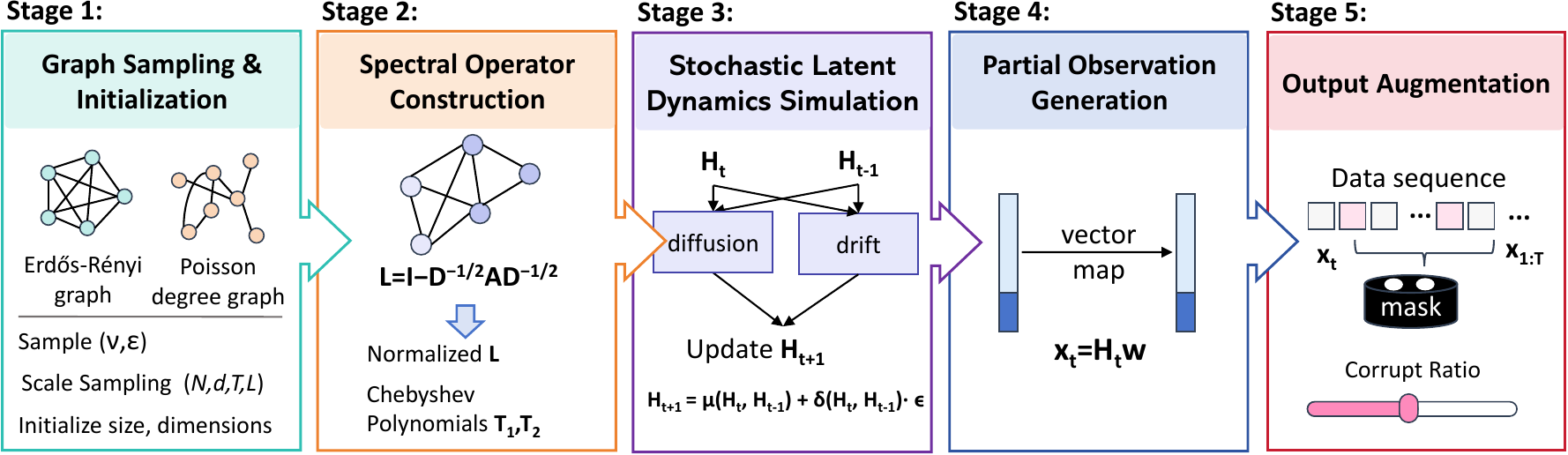}
    \vspace{-15pt}
    \caption{Synthetic data pipeline.}
    \vspace{-5pt}
    \label{fig:data pipeline}
\end{figure}

\noindent \textbf{Graph Sampling \& Operator Construction.}
To promote structural diversity, the graph topology is sampled on-the-fly from random graph families, including Erd\H{o}s–R\'enyi and Poisson degree models. 
Self-loops are added to stabilize the spectrum. 
Given $\mathbf{A}$, we compute the normalized Laplacian
$
\mathbf{L} = \mathbf{I} - \mathbf{D}^{-1/2} \mathbf{A} \mathbf{D}^{-1/2},
$
where $\mathbf{D}$ is the degree matrix. 
To model spatial interactions in a principled manner, we leverage Chebyshev polynomials of the rescaled Laplacian. 
Let $\widehat{\mathbf{L}}$ denote the spectrum-normalized Laplacian, and define
$
\mathbf{T}_1 = \widehat{\mathbf{L}}, \mathbf{T}_2 = 2\widehat{\mathbf{L}}^2 - \mathbf{I}.
$
These operators provide localized yet expressive spectral filters without requiring eigen-decomposition during sequence generation. 
To ensure scale generalization, we dynamically sample system configurations from broad distributions, including graph size $N$, latent dimension $d$, and horizons $(T, L)$, exposing the model to highly heterogeneous spatial scales and temporal regimes.

\noindent \textbf{Stochastic Latent Dynamics.}
The temporal evolution follows a stochastic recurrence:
\begin{equation}
\mathbf{H}_{t+1} 
= a \mathbf{H}_t 
+ \boldsymbol{\mu}(\mathbf{H}_t, \mathbf{H}_{t-1}, \mathbf{A})
+ \boldsymbol{\sigma}(\mathbf{H}_t, \mathbf{H}_{t-1}, \mathbf{A})
\odot \boldsymbol{\epsilon}_t 
+ \mathbf{b},
\end{equation}
where $a \in (0,1)$ controls system dissipation,
$\boldsymbol{\epsilon}_t \sim \mathcal{N}(\mathbf{0}, \mathbf{I})$ introduces stochastic excitation, 
$\mathbf{b}$ is a sample-wise bias term, and $\odot$ denotes element-wise multiplication. The drift term $\boldsymbol{\mu}$ and diffusion term $\boldsymbol{\sigma}$ are parameterized as
\begin{equation}
    \boldsymbol{\mu}(\cdot)
    = \tanh\!\Big(
    [\mathbf{T}_1 \mathbf{H}_t \;\Vert\; \mathbf{T}_2 \mathbf{H}_{t-1}]
    \mathbf{W}_1
    \Big)\mathbf{W}_2,
    \quad
    \boldsymbol{\sigma}(\cdot)=
    \mathrm{sigmoid}\!\Big(
    [\mathbf{T}_1 \mathbf{H}_t \;\Vert\; \mathbf{T}_2 \mathbf{H}_{t-1}]
    \mathbf{W}_3
    \Big),
\end{equation}
where $\mathbf{W}_1,\mathbf{W}_2,\mathbf{W}_3 \in \mathbb{R}^{d_h \times d_h}$ are learnable parameters of the generator and 
$[\cdot \Vert \cdot]$ denotes channel-wise concatenation. 
This formulation couples first- and second-order spectral responses with nonlinear transformations, producing rich spatio-temporal behaviors including diffusion-like smoothing, oscillatory propagation, and stochastic perturbations.
Notably, the generator parameters are randomly initialized for each sample, yielding a distribution over dynamical systems rather than a single fixed simulator. 
Consequently, the pre-training corpus spans a broad manifold of graph-coupled stochastic processes with heterogeneous spectral characteristics and stability regimes.

\noindent \textbf{Observation Generation.} Observable signals are obtained through a random linear projection
$
\mathbf{x}_t = \mathbf{H}_t \mathbf{w},
$
where $\mathbf{w} \in \mathbb{R}^{d}$ is sampled per instance. 
This projection induces partial observability, mimicking real-world sensing processes that capture low-dimensional measurements of higher-dimensional latent states. 
To approximate realistic data imperfections while maintaining full control of the generative mechanism, we randomly mask entries in $\mathbf{X}_{1:T}$ with a controllable ratio. 
Since these corruptions are externally imposed and independent of the underlying dynamics, the model can in principle disentangle intrinsic evolution patterns from observational artifacts.

\noindent \textbf{Output Augmentation.} To bridge the gap between ideal synthetic signals and imperfect real-world measurements, we randomly mask a proportion of the generated observable sequence with missing values. This simulates the missing data ubiquitous in real-world sensor networks. Crucially, while missing data in real-world datasets often exhibits fixed topological biases (e.g., broken sensors at specific locations), 
our random masking strategy enforces an unstructured, unbiased distribution of missingness across both space and time. This prevents the model from relying on superficial, dataset-specific failure topologies and instead forces it to learn robust, position-invariant spatiotemporal dependencies, significantly enhancing its zero-shot transferability to unseen downstream systems.

\section{Model Architecture}
\label{sec:model}

A Spatio-Temporal Foundation Model (STFM) is defined as a parametric function
$f_\theta : (\mathbf{X}_{1:T}, \mathbf{A}) \mapsto \widehat{\mathbf{Y}}_{T+1:T+L},$
where $\theta$ denotes the model parameters, 
$\mathbf{X}_{1:T} \in \mathbb{R}^{T \times N}$ is the historical observation, 
$\mathbf{A} \in \mathbb{R}^{N \times N}$ encodes spatial dependencies, 
and $\widehat{\mathbf{Y}}_{T+1:T+L} \in \mathbb{R}^{L \times N}$ is the predicted future sequence. 
Formally, an STFM should support arbitrary input lengths $T \in \mathbb{N}^+$, 
arbitrary graph sizes $N \in \mathbb{N}^+$, 
and arbitrary forecasting horizons $L \in \mathbb{N}^+$.
NeoST follows a latent-space spatio-temporal modeling paradigm that decouples representation learning, global reasoning, and output adaptation within a unified foundation backbone. 
Figure \ref{fig:model} illustrates the overall architecture of our model.
The model first transforms raw observations into structured patch-level embeddings,
producing a compact latent representation that captures local temporal continuity and node-level semantics. These patch embeddings are then processed by a latent \textbf{Spatio-Temporal Encoder} that jointly models temporal dependencies and graph-aware spatial interactions, yielding context-aware latent states. On top of the encoded representations, a \textbf{Hierarchical Latent Reasoner} generates future trajectory embeddings by performing multi-step global reasoning directly in latent space, enabling the model to construct and refine candidate futures without relying on strictly autoregressive decoding. Finally, a \textbf{Output Adapter} maps latent predictions back to the observation space to produce $\widehat{\mathbf{Y}}_{T+1:T+L}$, supporting both pre-training and tuning. All major components are cascaded through the Hyper-connection \cite{zhu2024hyper}, which preserves information flow across modules while maintaining stable optimization.

\begin{figure}[!h]
    \centering
    \includegraphics[width=\linewidth]{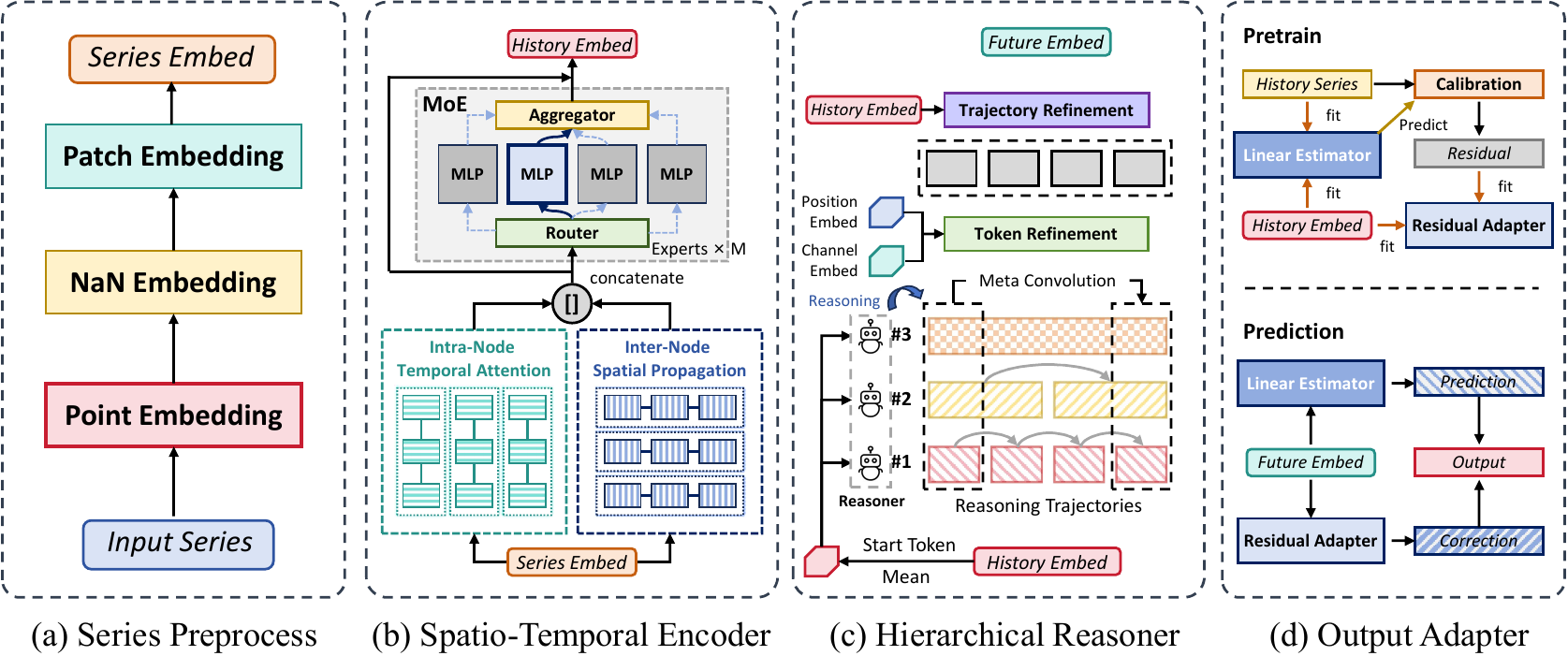}
    \vspace{-12pt}
    \caption{Overall architecture of NeoST.}
    \vspace{-12pt}
    \label{fig:model}
\end{figure}

\subsection{Preprocessing}

Given an input sequence $\mathbf{X}_{1:T} \in \mathbb{R}^{T \times N}$, we first apply instance normalization \cite{kim2021reversible} along the temporal dimension to reduce scale heterogeneity.
For each node $v_i$, we compute the temporal mean $\mu^{(i)}$ and standard deviation $\sigma^{(i)}$, and normalize the signal via $\hat{x}_t^{(i)} = (x_t^{(i)} - \mu^{(i)}) / \tilde{\sigma}^{(i)}$, where $\tilde{\sigma}^{(i)} = \max(\sigma^{(i)}, 1)$ prevents small-variance sequences from being over-amplified. Each normalized scalar observation is then lifted to a $d_p$-dimensional vector through a fixed random linear projection $\mathbf{z}_t^{(i)} = \hat{x}_t^{(i)} \mathbf{w} + \mathbf{b}$, where $\mathbf{w}$ and $\mathbf{b}$ are randomly initialized and kept non-trainable to stabilize optimization, and masked values are mapped to zero vectors. The resulting embeddings are segmented into temporal patches of length $p$ with stride $s$, optionally allowing overlap, and each flattened patch is projected to the hidden dimension $d$ to form patch tokens. 
An optional trend decomposition branch computes first-order differences between adjacent patches and injects the encoded residual signal into the patch representation, enabling sensitivity to short-range dynamics.

\subsection{Spatio-Temporal Encoder}

The spatio-temporal encoder transforms the patch-level representation $\mathbf{Z}_{1:P_T} \in \mathbb{R}^{P_T \times N \times d}$, where $P_T$ denotes the number of patches obtained from patchifying the input sequence, into a contextual latent memory. This is accomplished through alternating between two complementary mechanisms: intra-node temporal attention, which captures temporal evolution within individual nodes, and inter-node spatial propagation, which disseminates information across nodes according to learned relational affinities. The encoder stacks $L_{\mathrm{enc}}$ such layers, with each layer jointly implementing both mechanisms in a unified computation block. 
Critically, the two pathways are designed to be mutually exclusive in their information sources, ensuring that each mechanism contributes a distinct representation.

\noindent \textbf{Intra-node Temporal Attention.} For each node $v_i$, the patch sequence $\mathbf{Z}_{1:P_T}^{(i)} \in \mathbb{R}^{P_T \times d}$ is processed independently by a causal gate self-attention module \cite{qiu2025gated} equipped with Rotary Position Embeddings (RoPE) \cite{su2024roformer}, which encodes relative temporal positions without relying on additive positional encodings. This operation produces two complementary value representations per node $\mathbf{V}^{\mathrm{res}}$ and $\mathbf{V}^{\mathrm{dis}}$. $\mathbf{V}^{\mathrm{res}} \in \mathbb{R}^{P_T \times N \times d}$ captures the self-contained temporal dynamics of each node and serves as a residual pathway that preserves node-specific history. $\mathbf{V}^{\mathrm{dis}} \in \mathbb{R}^{P_T \times N \times d}$ produces temporal features that are subsequently made available for spatial distribution to neighboring nodes. The temporal pathway operates independently per node while the spatial pathway is explicitly prevented from collapsing into identity by masking self-connections, ensuring that each mechanism contributes a distinct and complementary inductive bias to the final representation.

\noindent \textbf{Inter-node Spatial Propagation.} Given the temporally encoded features, the encoder performs spatial information routing through a graph attention mechanism operating across nodes at each patch position. Queries $\mathbf{Q} \in \mathbb{R}^{P_T \times N \times d}$ and keys $\mathbf{K} \in \mathbb{R}^{P_T \times N \times d}$ are computed from the current patch representations, and a node-to-node attention map is derived as
$
    \mathbf{A}_{p} = \mathrm{softmax}\left(a \mathbf{Q}_p \mathbf{K}_p^\top + \mathbf{M}\right) \in \mathbb{R}^{N \times N},
$
where $a=1/\sqrt{d/H}$ is the scalar factor, $H$ is the number of attention heads, $\mathbf{M}$ is a masking matrix with $\mathbf{M}_{ii} = -\infty$ and $\mathbf{M}_{ij} = 0$ for $i \neq j$. The diagonal masking explicitly prohibits each node from attending to itself during spatial aggregation, ensuring that the spatial pathway carries exclusively cross-node information. The attention-weighted aggregation of $\mathbf{V}^{\mathrm{dis}}$ yields a spatially-propagated representation $\mathbf{V}^{\mathrm{agg}} = \mathbf{A}_p \mathbf{V}^{\mathrm{dis}}_p$, which is then modulated by a learned sigmoid gate $\mathbf{g} = \sigma(\mathbf{W}_g \mathbf{Z})$ to adaptively control how much spatial context is incorporated at each position.

\noindent \textbf{Output Combination.} The outputs of both pathways are concatenated and linearly projected to produce the updated representation, followed by a position-wise feed-forward network:
\begin{equation}
    \mathbf{Z}' = \mathbf{W}_o \left[ \mathbf{V}^{\mathrm{res}} \,\|\, \mathbf{V}^{\mathrm{agg}} \odot \mathbf{g} \right] + \mathbf{Z}, \qquad \mathbf{Z}'' = \mathrm{FFN}(\mathbf{Z}') + \mathbf{Z}'.
\end{equation}
This formulation ensures that each node's output integrates both its own temporal trajectory and a gated summary of information propagated from the rest of the graph. The final encoder output $\mathbf{Z}_{1:P_T} \in \mathbb{R}^{P_T \times N \times d}$ serves as the contextual memory that conditions all subsequent latent reasoning. Furthermore, the dense residual pathways embedded within this integration step facilitate stable gradient flow, guaranteeing robust representation learning during the end-to-end pre-training phase.

\subsection{Hierarchical Reasoner}

Given the contextual patch memory $\mathbf{Z}_{1:P_T} \in \mathbb{R}^{P_T \times N \times d}$ produced by the spatio-temporal encoder, the hierarchical reasoner is responsible for generating future trajectory embeddings directly in latent space. Rather than decoding predictions step-by-step in the observation domain, the reasoner operates entirely over patch tokens, producing $\hat{\mathbf{Z}}_{1:P_L} \in \mathbb{R}^{P_L \times N \times d}$, where $P_L$ is the number of output patch tokens corresponding to the target horizon $L$. 
The overall reasoning process follows a coarse-to-fine pipeline: 
An ensemble of multi-scale predictors first generates candidate trajectory hypotheses at different temporal resolutions in latent space. These candidates are then consolidated and refined at both the token level and the trajectory level, enabling the model to reconcile multi-scale information and align its latent predictions with the encoded historical context.

\noindent \textbf{Multi-Scale Latent Prediction.}
To capture future dynamics at multiple temporal resolutions, we instantiate a set of $S$ independent predictors $\{\mathcal{P}_s\}_{s=1}^{S}$, each operating at a distinct temporal stride $\Delta_s$ patch steps. Predictor $\mathcal{P}_s$ maps a global context summary $\bar{\mathbf{z}} \in \mathbb{R}^{N \times d}$ to a sequence of latent states at its native stride and 
truncated to $P_L$ steps, which can be written as 
$
\hat{\mathbf{Z}}^{(s)}_{1:P_L} = \mathcal{P}_s\!\left(\bar{\mathbf{z}},\; P_L\right) \in \mathbb{R}^{P_L \times N \times d}.
$
Here, $\bar{\mathbf{z}} = (1/P)\sum_{p=1}^{P} \mathbf{Z}_p$ condenses the average observed historical context into a compact initialization signal. The $S$ resulting trajectories are concatenated along the feature dimension to form a multi-scale composite representation
$
\hat{\mathbf{Z}}^{\mathrm{cat}}_{1:P_L} = \left[\hat{\mathbf{Z}}^{(1)} \,\Vert\, \hat{\mathbf{Z}}^{(2)} \,\Vert\, \cdots \,\Vert\, \hat{\mathbf{Z}}^{(S)}\right] \in \mathbb{R}^{P_L \times N \times Sd},
$
to which a learnable scale embedding $\mathbf{e}_{\mathrm{ch}} \in \mathbb{R}^{Sd}$ is added to distinguish each scale's contribution prior to fusion. In our implementation, each predictor $\mathcal{P}_s$ is realized as a ChebConv-LSTM \cite{seo2016structuredsequencemodelinggraph} with polynomial order $K$, which propagate spatial dependencies across nodes at every step.

\noindent \textbf{Token-Level Refinement via Conditional Temporal Convolution.}
The multi-scale composite is subsequently fused into a unified $d$-dimensional token sequence. Critically, the appropriate fusion strategy is not universal but depends on the specific input context.
We therefore formulate token-level refinement as a context-conditioned operation, where the fusion weights are dynamically generated from the observed context rather than fixed as shared parameters. A lightweight meta-network $\phi_{\mathrm{meta}}$ takes the global context summary $\bar{\mathbf{z}}$ and produces an instance-specific convolutional kernel
$
\mathbf{W}^{\mathrm{dyn}} = \phi_{\mathrm{meta}}(\bar{\mathbf{z}}) \in \mathbb{R}^{d \times (Sd + d) \times k},
$
where $k$ is the kernel width. This kernel is then applied along the temporal dimension to a augmented input $\left[\hat{\mathbf{Z}}^{\mathrm{cat}} + \mathbf{e}_{\mathrm{ch}} \;\Vert\; \mathbf{E}^{\mathrm{pos}}\right]$, which concatenates the channel embedding $\mathbf{e}_{\mathrm{ch}} \in \mathbb{R}^{Sd}$ with sinusoidal position embeddings $\mathbf{E}^{\mathrm{pos}} \in \mathbb{R}^{P_L \times d}$ to provide positional awareness, yielding $\hat{\mathbf{Z}}^{\mathrm{fused}} \in \mathbb{R}^{P_L \times N \times d}$. The convolution is applied independently for each node, ensuring that the token-level merging of multi-scale signals is modulated by the observed context.

\noindent \textbf{Trajectory-Level Refinement via Encoder-Decoder Reasoning.}
While token-level fusion reconciles scale inconsistencies locally, the resulting trajectory may still exhibit global structural incoherence or misalignment with the historical context. To address this, we apply a two-stage refinement that reasons over the full trajectory sequence jointly. A trajectory encoder $\mathcal{R}_{\mathrm{enc}}$ first processes $\hat{\mathbf{Z}}^{\mathrm{fused}}$ to propagate long-range temporal dependencies across future tokens, after which a trajectory decoder $\mathcal{R}_{\mathrm{dec}}$ grounds the refined trajectory in the observed history $\mathbf{Z}_{1:P_T}$ through cross-attention.
Both $\mathcal{R}_{\mathrm{enc}}$ and $\mathcal{R}_{\mathrm{dec}}$ are instantiated as stacks of transformer layers.
$\mathcal{R}_{\mathrm{enc}}$ employs causal gate self-attention with RoPE to model temporal ordering within the predicted trajectory, while $\mathcal{R}_{\mathrm{dec}}$ performs cross gate attention between the refined trajectory queries and $\mathbf{Z}_{1:P_T}$.

\section{Training \& Inference}
\label{sec:training and inference}
\noindent \textbf{Pre-training Phase.}
Beyond providing ground-truth future values $\mathbf{Y}_{T+1:T+L}$, the engine exposes the underlying latent state vectors $\{\mathbf{H}_t\}_{t=T+1}^{T+L}$ that directly govern the generative dynamics of each synthetic system.
This privileged supervisory signal, unavailable from real-world datasets, forms the cornerstone of our pre-training objective.
Concretely, we flatten $\hat{\mathbf{Z}}^{\mathrm{out}}$ to $\hat{\mathbf{Z}}^{\mathrm{flat}} \in \mathbb{R}^{(P_L N) \times d}$ and apply the Principal Component Analysis (PCA) \cite{wold1987principal} to project it into the $d_h$-dimensional space of the synthetic latent matrix $\mathbf{H} \in \mathbb{R}^{(P_L N) \times d_h}$. The resulting pre-training objective, which we term \textbf{Latent Alignment Loss (LAL)}, interpolates between a latent alignment term and a standard observation-space reconstruction term via a scalar hyperparameter $\alpha \in [0,1]$:
\begin{equation}\label{eq:sls}
\mathcal{L}_{\mathrm{LAL}} = (1-\alpha)\,\bigl\|\mathbf{H} - \mathrm{PCA}_{d_h}\!\bigl(\hat{\mathbf{Z}}^{\mathrm{flat}}\bigr)\bigr\|_F^2 + \alpha\,\bigl\|\mathbf{Y}_{T+1:T+L} - \hat{\mathbf{Y}}_{T+1:T+L}\bigr\|_F^2.
\end{equation}
By placing the primary supervisory weight on the latent space, $\mathcal{L}_{\mathrm{LAL}}$ steers NeoST toward internalizing the structural and dynamical properties of spatio-temporal systems rather than overfitting to the superficial statistical patterns of any particular observation domain.

\noindent \textbf{Test-Time Adaptation.}
Let $\tilde{\mathbf{Z}}^{\mathrm{enc}} \in \mathbb{R}^{T \times N \times d_s}$ denote the time-resolved encoder embeddings obtained by depatchifying the encoder output, where $d_s = d/p$ is the per-timestep embedding dimension and $p$ is the patch length. Collecting all spatiotemporal tokens that are not masked or missing, we form the feature matrix $\mathbf{Z}_{\mathrm{obs}} \in \mathbb{R}^{M \times d_s}$ and the corresponding observation vector $\mathbf{v}_{\mathrm{obs}} \in \mathbb{R}^{M}$, where $M$ is the number of valid observations. The adapter solves the ridge regression problem by utilizing the following equation:
$
\mathbf{w}^{*}, b^{*} = \mathop{\arg\min}_{\mathbf{w},\, b}\;\bigl\|\mathbf{Z}_{\mathrm{obs}}\mathbf{w} + b\,\mathbf{1} - \mathbf{v}_{\mathrm{obs}}\bigr\|^2 + \lambda\|\mathbf{w}\|^2.
$
The equation admits the closed-form solution $\mathbf{w}^{*} = (\mathbf{Z}_{\mathrm{obs}}^{\top}\mathbf{Z}_{\mathrm{obs}} + \lambda\mathbf{I})^{-1}\mathbf{Z}_{\mathrm{obs}}^{\top}\tilde{\mathbf{v}}_{\mathrm{obs}}$, where $\tilde{\mathbf{v}}_{\mathrm{obs}}$ denotes the bias-corrected target. Denoting the depatchified future embeddings as $\tilde{\mathbf{Z}}^{\mathrm{out}} \in \mathbb{R}^{L \times N \times d_s}$, the predicted future sequence is obtained by applying the fitted mapping
$
\hat{\mathbf{Y}}_{T+1:T+L} = \tilde{\mathbf{Z}}^{\mathrm{out}}_{\mathrm{flat}}\,\mathbf{w}^{*} + b^{*},
$
where $\tilde{\mathbf{Z}}^{\mathrm{out}}_{\mathrm{flat}} \in \mathbb{R}^{(LN) \times d_s}$ is the reshaped future embedding matrix.

\section{Experiments}
\label{sec:experiments}

In our experiments, we aim to address the following research questions (RQs):
\vspace{-2pt}
\begin{itemize}[leftmargin=*]
    \item \textbf{(RQ1)}: How effectively does NeoST transfer to diverse unseen real-world systems?
    \item \textbf{(RQ2)}: How does each component and hyperparameter in NeoST contribute to model performance?
    \item \textbf{(RQ3)}: How efficient is NeoST compared to other foundation models?
    \item \textbf{(RQ4)}: How can synthetic spatio-temporal data enhance the generalization ability of models?
\end{itemize}

\subsection{Experimental Setups} 
\textbf{Datsets \& Baselines.} NeoST's pre-training utilizes the synthetic data introduced in Section \ref{sec:data}. For testing, we selected six classic spatio-temporal datasets for evaluation: PEMS03/04/07/08 \cite{guo2021learning}, PEMS-BAY \& METR-LA \cite{li2018dcrnn_traffic}. 
We compare NeoST with the following models: 
\vspace{-2pt}
\begin{itemize}[leftmargin=*]
    \item \textbf{Spatio-Temporal Foundation Models (STFMs)}: FactoST \cite{zhonglearning}, OpenCity \cite{li2024opencity}, UniST \cite{yuan2024unist}
    \item \textbf{Time Series Foundation Models (TSFMs)}: ROSE \cite{wang2024rose}, TimesFM \cite{Das2024TimesFM}, Moirai \cite{liu2025moirai};
\end{itemize}
\vspace{-2pt}
We adopted Mean Absolute Error (MAE) and Root Mean Square Error (RMSE) as evaluation metrics.

\noindent \textbf{Implementation Details.} 
All experiments are conducted on a single NVIDIA A100 80GB GPU paired with an Intel Xeon Platinum 8358P CPU (32C @ 2.60GHz). During the synthetic data generation process, the number of nodes and the latent dimension for each spatio-temporal graph are uniformly sampled from the ranges $[64, 512]$ and $[4, 16]$, respectively, producing 1,000 distinct samples per epoch. The core model architecture is configured with a hidden dimension of 96 and 4 attention heads. The spatio-temporal encoder stacks 4 layers, while the trajectory refinement module is structured with a 2-layer encoder and a 2-layer decoder. The hyper-connections throughout the network maintain an expansion rate of 4. We pre-train the model using AdamW \cite{loshchilov2017decoupled} with a batch size of 8 for 15 epochs. The other experimental settings are the same as those in \citet{zhonglearning}.

\subsection{Model Comparison (RQ1)}

We evaluate the models on six established spatio-temporal datasets in the zero-shot scenario, with the results summarized in Table \ref{tab:zero_shot_comparison}. NeoST consistently surpasses both spatio-temporal and time series foundation models, achieving the lowest MAE and RMSE in ten out of the twelve evaluated settings. Our model demonstrates particularly substantial improvements in challenging long-horizon scenarios of 96 steps, lowering the MAE by significant margins on datasets such as PEMS-03, PEMS-04, PEMS-07, and METR-LA compared to the strongest baselines. This sustained long-term accuracy validates the effectiveness of our latent reasoning paradigm in mitigating the error accumulation typically observed in conventional autoregressive architectures. Although FactoST secures a marginal lead on the short horizon of PEMS-Bay and TimesFM performs slightly better on the long horizon of PEMS-08, NeoST remains highly competitive in these instances while dominating the overall benchmark. Crucially, these strong zero-shot results confirm that pre-training exclusively on diverse synthetic data successfully circumvents the distributional biases of real-world datasets.
\newpage
\begin{table*}[h!]
  \centering
  \caption{Quantitative comparison of zero-shot forecasting performance. Results are reported as MAE / RMSE, where lower values indicate better predictive performance. The best and second-best results are highlighted in \boldres{red} and \secondres{blue}, respectively. Datasets marked with an asterisk (*) were included in the pre-training corpora of specific baselines and are therefore excluded from the ranking.}
  \setlength{\tabcolsep}{4.5pt}
  \footnotesize
  \renewcommand{\arraystretch}{1.0}
  \resizebox{\linewidth}{!}{
    \begin{tabular}{lc|c|ccc|ccc}
      \toprule
      \multicolumn{2}{c|}{\textbf{Model Type}} 
        & \multicolumn{1}{c|}{\textbf{Ours}}
        & \multicolumn{3}{c|}{\textbf{STFMs}}
        & \multicolumn{3}{c}{\textbf{TSFMs}} \\
      \cmidrule{1-2}\cmidrule{3-3}\cmidrule{4-6}\cmidrule{7-9}
      \textbf{Dataset} & \textbf{Horizon}
        & \textbf{NeoST}
        & \textbf{FactoST}
        & \textbf{OpenCity}
        & \textbf{UniST}
        & \textbf{ROSE}
        & \textbf{TimesFM}
        & \textbf{Moirai} \\
      \midrule
      \multirow{2}{*}{PEMS-03}
       & 12→12 & \boldres{25.16} / \boldres{34.90} & 30.12 / \secondres{46.92} & 30.37 / 47.49 & 101.87 / 129.94 & 29.74 / 47.27 & \secondres{29.71} / 49.21 & 28.23* / 46.36* \\
       & 96→96 & \boldres{84.49} / \boldres{107.85} & 113.62 / 144.80 & 125.18 / 159.23 & \secondres{102.87} / \secondres{137.90} & 122.93 / 158.76 & 108.59 / 152.74 & 74.85* / 110.78* \\
      \midrule
      \multirow{2}{*}{PEMS-04}
       & 12→12 & \boldres{30.98} / \boldres{43.63} & 38.65 / 56.59 & 39.34* / 58.50* & 67.91 / 87.63 & 38.35 / 57.34 & \secondres{35.00} / \secondres{53.27} & 34.65* / 52.13* \\
       & 96→96 & \boldres{105.25} / \boldres{137.35} & 142.11 / 175.99 & 153.18* / 188.95* & \secondres{115.76} / \secondres{153.58} & 151.39 / 189.29 & 127.39 / 171.13 & 105.95* / 141.40* \\
      \midrule
      \multirow{2}{*}{PEMS-Bay}
       & 12→12 & \secondres{2.67} / \secondres{4.80} & \boldres{2.02} / \boldres{4.59} & 3.23* / 6.91* & 14.89 / 16.92 & 2.72 / 6.41 & 6.59 / 14.99 & 1.97* / 4.69* \\
       & 96→96 & \boldres{4.86} / \boldres{8.14} & \secondres{6.12} / \secondres{10.47} & 6.92* / 11.79* & 9.29 / 12.72 & 7.02 / 11.90 & 14.16 / 24.01 & 6.51* / 11.70* \\
      \midrule
      \multirow{2}{*}{PEMS-07}
       & 12→12 & \boldres{35.70} / \boldres{49.48} & 45.47 / 66.68 & 45.18 / 67.15 & 104.93 / 133.65 & 44.79 / 67.33 & \secondres{42.10} / \secondres{64.59} & 35.64* / 50.25* \\
       & 96→96 & \boldres{144.18} / \boldres{175.65} & 156.07 / 190.59 & 172.20 / 211.25 & \secondres{144.67} / \secondres{184.62} & 172.07 / 212.94 & 151.99 / 205.18 & 125.91* / 169.43* \\
      \midrule
      \multirow{2}{*}{PEMS-08}
       & 12→12 & \boldres{26.28} / \boldres{36.42} & 32.14 / 47.27 & 32.45* / 48.42* & 73.46 / 93.14 & 31.70 / 47.77 & \secondres{29.68} / \secondres{45.18} & 38.23* / 53.12* \\
       & 96→96 & \secondres{94.03} / \boldres{121.87} & 121.43 / 151.59 & 128.48* / 161.75* & 104.77 / 136.21 & 129.05 / 163.15 & \boldres{92.72} / \secondres{126.10} & 119.10* / 151.83* \\
      \midrule
      \multirow{2}{*}{METR-LA}
       & 12→12 & \boldres{4.90} / \boldres{8.21} & 5.98 / 14.26 & 4.30* / 8.37* & 24.33 / 29.31 & 6.05 / 14.36 & 6.59 / 14.99 & \secondres{5.55} / \secondres{13.79} \\
       & 96→96 & \boldres{8.31} / \boldres{13.70} & \secondres{12.68} / \secondres{18.82} & 10.85* / 17.60* & 25.88 / 30.19 & 12.96 / 23.61 & 14.16 / 24.01 & 12.87 / 22.19 \\
      \bottomrule
    \end{tabular}
  }
  \label{tab:zero_shot_comparison}
  \vspace{-10pt}
\end{table*}

\subsection{Ablation and Hyperparameter Analysis (RQ2)}
To rigorously validate the individual contributions of NeoST's core components, we conduct a comprehensive ablation study. Specifically, we evaluate the zero-shot forecasting Mean Absolute Error (MAE) under both short-term (12 steps) and long-term (96 steps) horizons on the PEMS03 and PEMS07 datasets. This analysis systematically ablates critical architectural modules, examines the sensitivity of the latent alignment objective by varying the trade-off hyperparameter $\alpha \in [0, 1]$ (Figure \ref{fig:ablation1_rq2}), and isolates the empirical efficacy of the test-time adaptation mechanism (Figure \ref{fig:ablation2_rq2}).
\begin{figure}[!h]
    \centering
    \includegraphics[width=\linewidth]{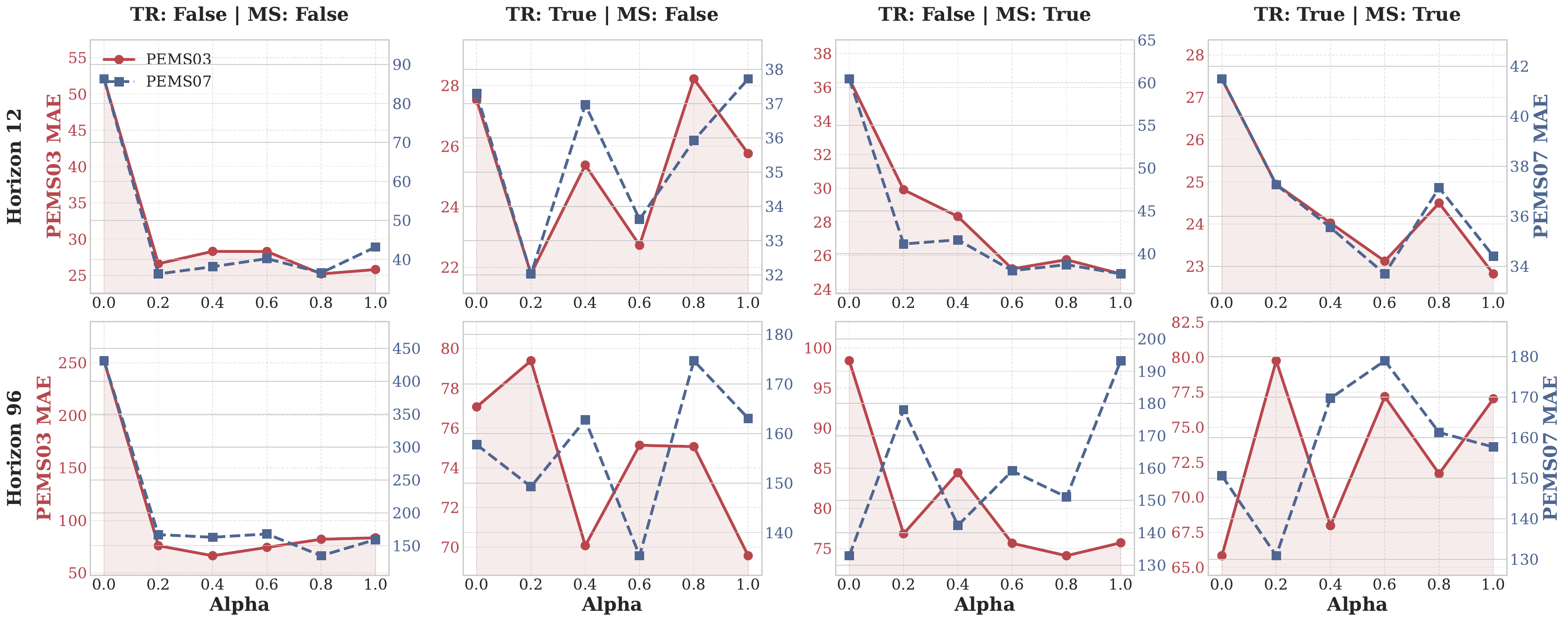}
    \vspace{-15pt}
    \caption{Ablation and hyperparameter analysis.}
    \vspace{-10pt}
    \label{fig:ablation1_rq2}
\end{figure}
\begin{itemize}[leftmargin=*]
    \item \textbf{Architectural Contributions:} The full NeoST framework attains the lowest MAE and exhibits superior predictive stability across most evaluated scenarios. Specifically, disabling multi-scale prediction significantly impairs forecasting robustness over long-term horizons, whereas omitting the trajectory refinement module results in a substantial increase in cumulative structural errors.
    \item \textbf{Objective Trade-off ($\alpha$):} Model predictive performance demonstrates high sensitivity to the balance between latent and observation-space objectives. We find that relying exclusively on latent alignment ($\alpha=0.0$) results in suboptimal error rates, while pure observation-space optimization ($\alpha=1.0$) restricts robust generalization under distributional shifts. The optimal balance consistently resides within $\alpha \in [0.2, 0.6]$, validating our proposed synergistic dual-objective design.
    \item \textbf{Test-Time Adaptation (TTA):} As shown in Figure  \ref{fig:ablation2_rq2}, disabling the TTA module consistently degrades forecasting accuracy across  most scenarios. The performance drop is particularly pronounced in both short and long horizons on PEMS07 and PEMS08. This confirms that leveraging inference-time observation context is crucial for mitigating distribution shifts in zero-shot transfer.
\end{itemize}

\begin{figure}[!h]
    \centering
    \includegraphics[width=\linewidth]{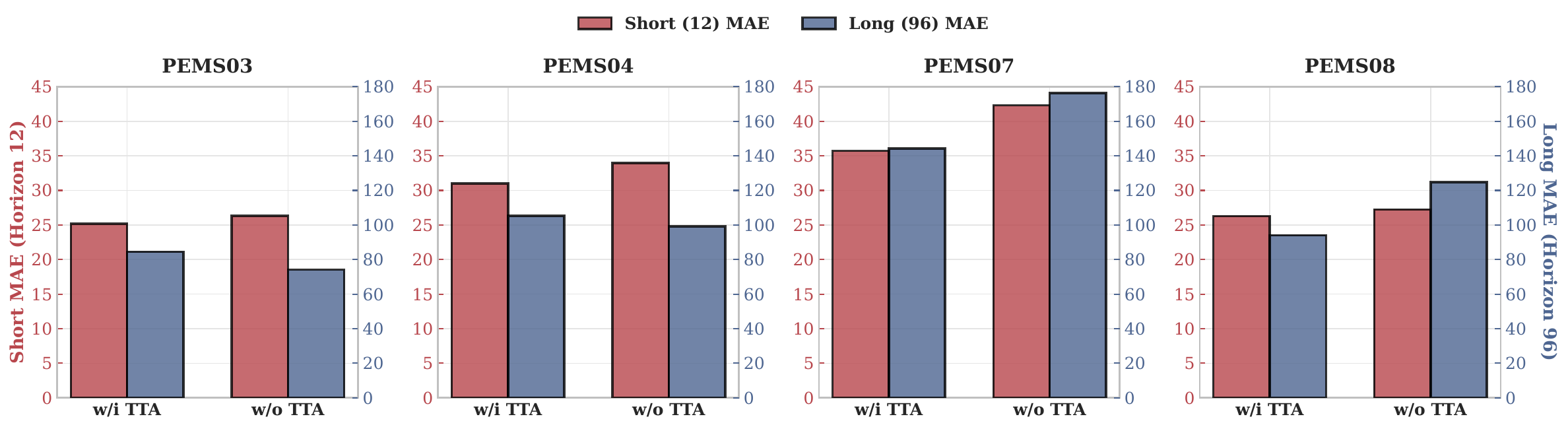}
    \vspace{-15pt}
    \caption{Impact of Test-Time Adaptation (TTA).}
    \vspace{-15pt}
    \label{fig:ablation2_rq2}
\end{figure}

\newpage

\subsection{Inference Efficiency (RQ3)}

\begin{wrapfigure}{l}{0.5\linewidth}
    \centering
    \vspace{-13pt}
    \includegraphics[width=\linewidth]{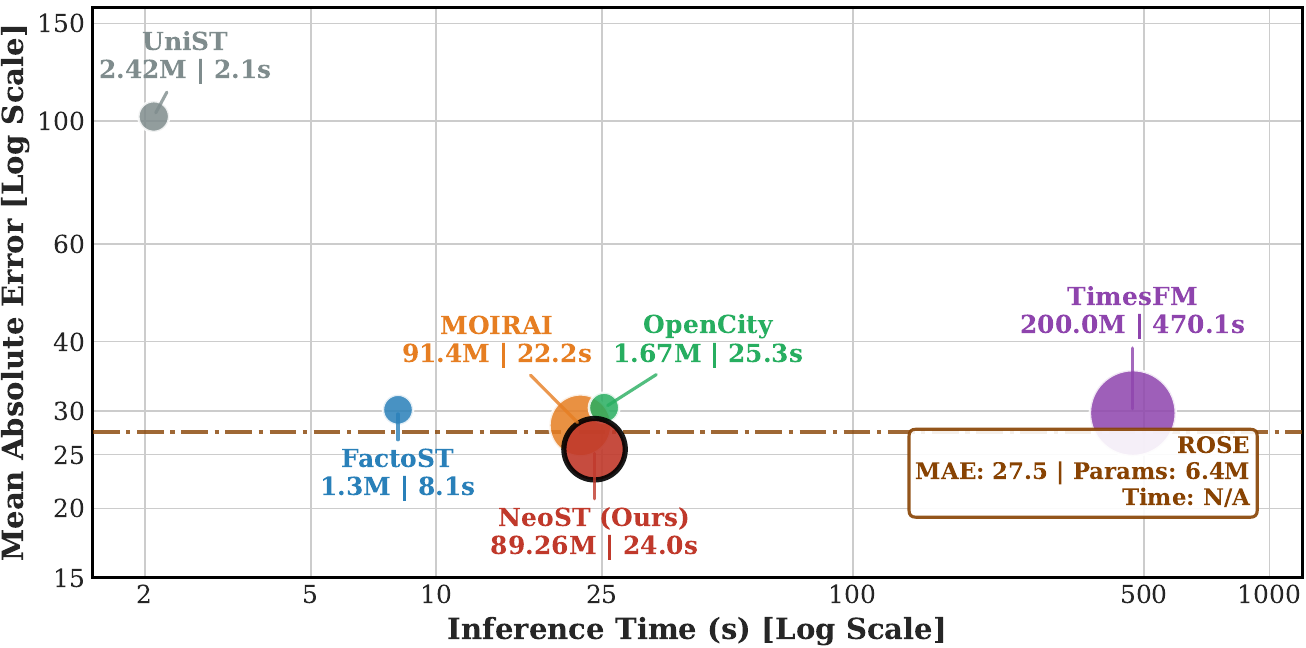}
    \vspace{-15pt}
    \caption{Efficiency comparison.}
    \vspace{-10pt}
    \label{fig:efficiency}
\end{wrapfigure}

To answer RQ3, we evaluate the computational efficiency of NeoST by comparing its inference time, parameter scale, and predictive performance (MAE) against representative baselines on PEMS03 dataset, as illustrated in Figure \ref{fig:efficiency}. NeoST achieves the lowest MAE among all evaluated models, demonstrating superior predictive accuracy. While lightweight models like UniST (2.1s) and FactoST (8.1s) exhibit faster inference speeds, they suffer from significantly higher prediction errors. Conversely, massive models such as TimesFM (200.0M params, 470.1s) incur prohibitive computational costs without outperforming our approach. Operating with 89.26M parameters and a 24.0s inference time, NeoST maintains a computational footprint comparable to MOIRAI (22.2s) and OpenCity (25.3s) but delivers substantially better forecasting accuracy. 
These results establish that NeoST achieves a favorable trade-off between reasoning efficiency and spatio-temporal modeling performance.

\subsection{Manifold Coverage of Synthetic Data (RQ4)}
\begin{wrapfigure}{r}{0.45\linewidth}
    \centering
    \vspace{-18pt}
    \includegraphics[width=\linewidth]{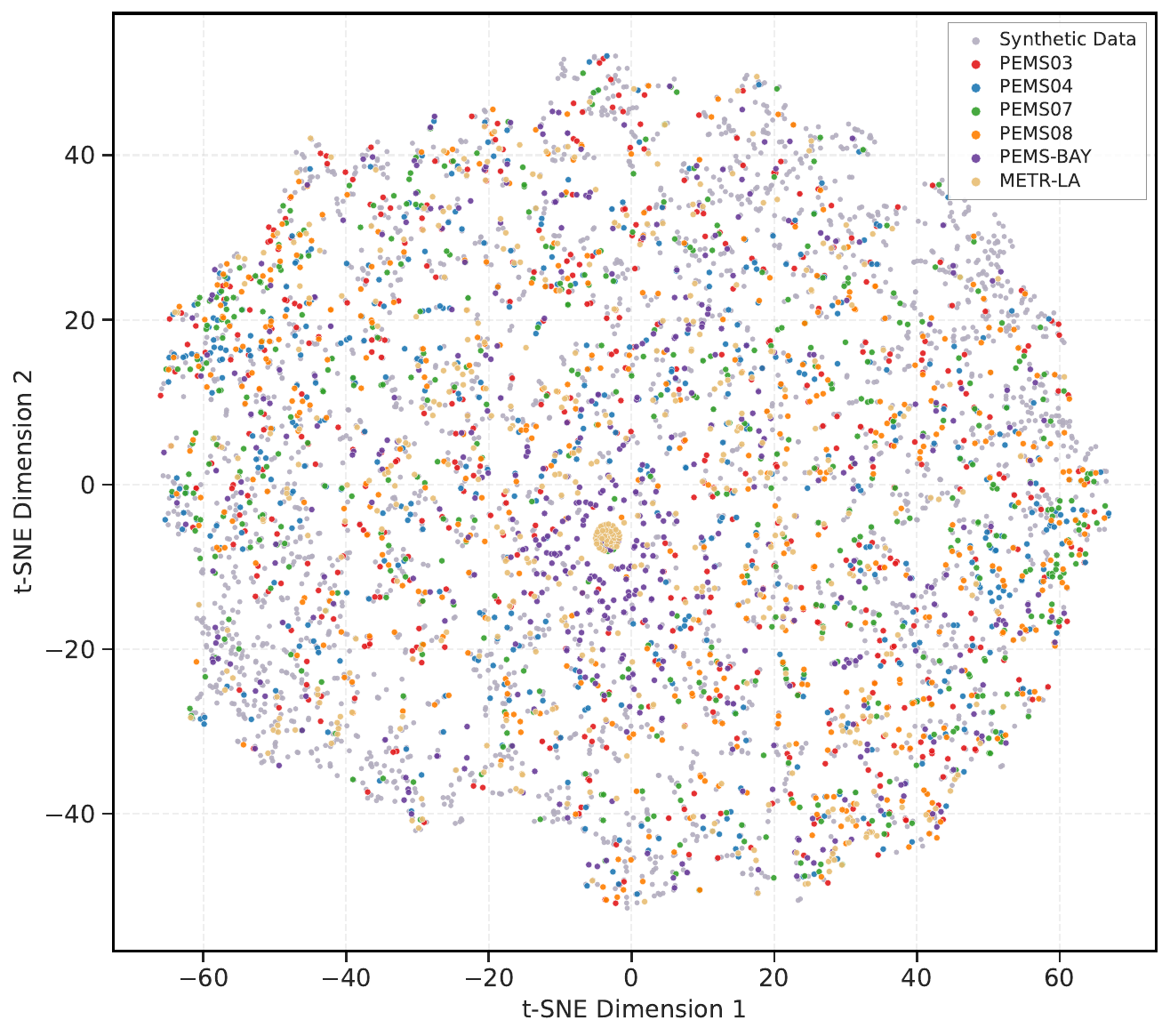}
    \vspace{-18pt}
    \caption{t-SNE visualization.}
    \vspace{-15pt}
    \label{fig:visualization}
\end{wrapfigure}
To address RQ4, we visualize the representational manifold of our synthetic corpus against six real-world benchmarks using t-SNE. After applying instance normalization to fixed-length temporal sequences (12 steps) from all datasets to isolate intrinsic temporal dynamics, we projected these features into a 2D space. As illustrated in Figure \ref{fig:visualization}, the synthetic data forms an expansive, continuous manifold that comprehensively envelopes the localized clusters of all real-world test sets. 
This empirical evidence validates our core hypothesis: pre-training exclusively on this highly diverse synthetic manifold equips the foundation model with universal structural priors, mitigating real-world data biases and enabling robust zero-shot generalization.

\section{Conclusion \& Future Works}

In this work, we introduced NeoST, a novel Spatio-Temporal Foundation Model that overcomes the limitations of biased real-world datasets and computationally expensive generation paradigms. By pre-training exclusively on a procedurally generated synthetic corpus, NeoST effectively mitigates distributional biases and measurement artifacts. 
Extensive experiments across diverse benchmarks demonstrate that NeoST achieves superior generalization, long-horizon stability, and inference efficiency in diverse real-world spatio-temporal systems.
For future work, we plan to scale the synthetic generation engine to incorporate more complex physical laws and non-stationary dynamics.

\newpage

\bibliographystyle{IEEEtranN}
\bibliography{references.bib}

\newpage
\appendix
\section{Related Works}
\label{appendix:related works}

\paragraph{Synthetic Spatio-Temporal Data.} 
The use of synthetic data for time series and spatio-temporal modeling has gained increasing attention, particularly in the era of foundation models.
A comprehensive survey is provided by \citet{liu2025empowering}, which highlights procedural and simulation-based generation as scalable alternatives to real-world data collection.
In contrast, most existing spatio-temporal models are trained on a limited set of real-world benchmarks, predominantly from transportation and urban domains, such as PEMS \cite{guo2021learning}, LargeST \cite{liu2023largest}, SINPA \cite{zhang2024predicting}, and LaDe \cite{wu2024lade}. While influential, these datasets share similar sensing infrastructures and domain biases, restricting distributional diversity.
Extending synthesis to the spatio-temporal domain introduces additional complexity, as models must jointly capture spatial dependencies and temporal dynamics. GAN-based approaches \cite{gao2022generative} generate realistic mobility traces or sensor readings but are typically domain-specific. PDE-based simulators \cite{arndt2025synthetic} construct synthetic spatio-temporal graphs governed by physical dynamics, offering improved controllability. Diffusion models for trajectory \cite{zhu2023difftraj,zhu2024controltraj} and time series \cite{xia2025causal} generation have also been explored.
Despite these efforts, synthetic data is rarely used as the \textit{primary} source for STFM pre-training. Most existing STFMs still rely heavily on aggregated real-world datasets, inevitably inheriting their inherent domain biases and persistent data quality limitations.

\paragraph{Training Paradigms and Objectives for Foundation Models.} 
Traditional point-wise regression losses in the observation space often struggle to capture global structural similarities or temporal dynamics, prompting the use of specialized alignment metrics~\cite{jadon2024comprehensive, cuturi2017soft, le2019shape}. In parallel, Self-Supervised Learning (SSL)~\cite{zhang2024self} has transformed representation learning. A particularly promising evolution is the Joint-Embedding Predictive Architecture (JEPA)~\cite{assran2023self, balestriero2025lejepa}. By forcing models to predict missing representations entirely within an abstract latent space rather than reconstructing raw values, JEPA inherently discards high-frequency noise and focuses on semantic features. This paradigm has recently achieved notable success in spatio-temporal tasks, such as trajectory representation learning via T-JEPA~\cite{li2024t} and HiT-JEPA~\cite{li2025hit}. These developments conceptually echo our proposed Latent Alignment Loss (LAL) in NeoST, reinforcing the necessity of latent-space reasoning for complex dynamics.
Furthermore, the shift toward foundation-level modeling is fundamentally intertwined with neural scaling laws. Recent empirical studies demonstrate that TSFMs adhere to predictable scaling behaviors, where performance consistently improves alongside increases in model parameters and dataset size~\cite{yao2024towards, shi2024scaling, edwards2024scaling}. By coupling a highly scalable synthetic data engine with efficient latent-space reasoning, NeoST is inherently positioned to exploit these scaling principles, circumventing the finite scale and inherent biases of real-world spatio-temporal datasets.

\section{Theoretical Properties of Synthetic Data}
\label{appendix:theory}

This appendix characterizes the dynamical properties induced by the proposed synthetic data
generator. Rather than instantiating a single fixed dynamical law, the generator samples graph
topologies, generator parameters, biases, observation maps, and masking patterns, thereby defining a
distribution over stochastic graph-coupled dynamical systems. Consequently, the synthetic corpus can
contain a broad range of spatio-temporal regimes, including spatial dependence, temporal
autocorrelation, delayed feedback, diffusion-like smoothing, wave-like propagation, quasi-periodic
responses, nonlinear saturation, heteroscedastic stochasticity, and partial observability.

Let $\Theta$ denote the sampled system configuration for one synthetic instance. Unless otherwise
stated, all statements below are conditioned on a fixed $\Theta$.

\paragraph{Distribution over dynamical systems.}
Since each synthetic instance samples its own graph, generator weights, bias, and observation
projection, the overall data distribution can be written as
\begin{equation}
p(\mathbf X)=\int p_{\Theta}(\mathbf X)\,d\nu(\Theta),
\end{equation}
where $\nu$ is the sampling distribution over system configurations. This mixture interpretation is
central to the design of the generator: different sampled systems may exhibit different spatial
scales, temporal responses, stochastic intensities, nonlinear regimes, and observation patterns.
Thus, the corpus covers a heterogeneous family of structured spatio-temporal processes rather than
trajectories from a single simulator.

\begin{proposition}[Conditional second-order Markov structure]
\label{prop:conditional_markov}
For a fixed sampled configuration $\Theta$, the latent process satisfies
\begin{equation}
p(\mathbf H_{t+1}\mid \mathbf H_t,\mathbf H_{t-1},\ldots,\mathbf H_0,\Theta)
=
p(\mathbf H_{t+1}\mid \mathbf H_t,\mathbf H_{t-1},\Theta).
\end{equation}
Equivalently,
\begin{equation}
\widetilde{\mathbf H}_t=
\begin{bmatrix}
\mathbf H_t\\
\mathbf H_{t-1}
\end{bmatrix}
\end{equation}
forms a first-order time-homogeneous Markov chain conditioned on $\Theta$.
\end{proposition}

\begin{proof}
Conditioned on $\Theta$, the one-step update depends on the past only through
$(\mathbf H_t,\mathbf H_{t-1})$ and the fresh innovation $\boldsymbol{\epsilon}_t$, which is
independent of previous states. The augmented-state representation then gives a first-order Markov
chain.
\end{proof}

The conditioning on $\Theta$ is essential. Since generator parameters are sampled per instance, the
marginal process obtained after integrating out $\Theta$ is a mixture of Markov processes and need
not itself be second-order Markov.

\begin{proposition}[Conditionally Gaussian heteroscedastic transition]
\label{prop:conditional_gaussian}
For fixed $\Theta$, the transition distribution satisfies
\begin{equation}
\mathrm{vec}(\mathbf H_{t+1})
\mid
\mathbf H_t,\mathbf H_{t-1},\Theta
\sim
\mathcal N(\mathbf m_t,\mathbf \Sigma_t),
\end{equation}
where
\begin{equation}
\mathbf m_t
=
\mathrm{vec}\!\left(
a\mathbf H_t+
\boldsymbol{\mu}(\mathbf H_t,\mathbf H_{t-1},\mathbf A)
+\mathbf b
\right),
\end{equation}
and
\begin{equation}
\mathbf \Sigma_t
=
\operatorname{diag}
\left(
\mathrm{vec}
\left(
\boldsymbol{\sigma}(\mathbf H_t,\mathbf H_{t-1},\mathbf A)^2
\right)
\right).
\end{equation}
Thus, the latent dynamics are conditionally Gaussian with state-dependent covariance.
\end{proposition}

\begin{proof}
Given $(\mathbf H_t,\mathbf H_{t-1})$ and $\Theta$, the only random component in the update is
\[
\boldsymbol{\sigma}(\mathbf H_t,\mathbf H_{t-1},\mathbf A)\odot\boldsymbol{\epsilon}_t.
\]
Since $\boldsymbol{\epsilon}_t$ has independent standard Gaussian entries, element-wise scaling
yields a diagonal Gaussian covariance.
\end{proof}

Although each one-step transition is conditionally Gaussian, the process is generally not marginally
Gaussian, because both its conditional mean and conditional variance depend nonlinearly on
graph-filtered latent states. This state-dependent variance introduces heteroscedastic
stochasticity, a common feature of real spatio-temporal sensing systems.

\begin{theorem}[Spectrally bounded graph filtering]
\label{thm:chebyshev_stability}
Assume the sampled graph is undirected with nonnegative weights and nonzero degrees. Then
\begin{equation}
\operatorname{spec}(\widehat{\mathbf L})\subset[-1,1],
\end{equation}
and for any Chebyshev polynomial $T_k$,
\begin{equation}
\|T_k(\widehat{\mathbf L})\|_2\leq 1.
\end{equation}
In particular,
\begin{equation}
\|\mathbf T_1\|_2\leq 1,
\qquad
\|\mathbf T_2\|_2\leq 1.
\end{equation}
\end{theorem}

\begin{proof}
For an undirected graph with nonnegative weights, the normalized Laplacian has spectrum in
$[0,2]$. Hence $\widehat{\mathbf L}=\mathbf L-\mathbf I$ has spectrum in $[-1,1]$. Since
$|T_k(x)|\leq 1$ for $x\in[-1,1]$, symmetry of $\widehat{\mathbf L}$ gives
\[
\|T_k(\widehat{\mathbf L})\|_2
=
\max_{\lambda\in\operatorname{spec}(\widehat{\mathbf L})}
|T_k(\lambda)|
\leq 1.
\]
\end{proof}

Theorem~\ref{thm:chebyshev_stability} concerns the stability of the graph filtering operators. At
the same time, $\mathbf T_1\mathbf H_t$ and $\mathbf T_2\mathbf H_{t-1}$ explicitly couple
neighboring or spectrally related nodes. Therefore, spatial dependence, graph smoothness, and
localized propagation are structural properties induced by the generator.

\begin{theorem}[Dissipative moment control]
\label{thm:moment_control}
Fix a sampled configuration $\Theta$. If the initial latent states have finite second moments, then
there exists $C_{\Theta}<\infty$ such that
\begin{equation}
\sup_{t\geq 0}
\mathbb E_{\Theta}\|\mathbf H_t\|_F^2
\leq
C_{\Theta}.
\end{equation}
\end{theorem}

\begin{proof}
Since $\tanh(\cdot)$ is bounded and the sampled weights are finite-dimensional, there exists
$M_{\mu,\Theta}<\infty$ such that
\[
\|\boldsymbol{\mu}(\mathbf H_t,\mathbf H_{t-1},\mathbf A)\|_F
\leq
M_{\mu,\Theta}
\]
for all states. Moreover, because $\mathrm{sigmoid}(z)\in(0,1)$,
\begin{equation}
\mathbb E_{\Theta}
\left[
\|\boldsymbol{\sigma}(\mathbf H_t,\mathbf H_{t-1},\mathbf A)
\odot
\boldsymbol{\epsilon}_t\|_F^2
\mid
\mathcal F_t
\right]
\leq
Nd_h.
\end{equation}
Choose $\eta>0$ such that $\alpha=(1+\eta)a^2<1$. Then
\begin{equation}
\mathbb E_{\Theta}
\left[
\|\mathbf H_{t+1}\|_F^2
\mid
\mathcal F_t
\right]
\leq
\alpha\|\mathbf H_t\|_F^2+C_{0,\Theta},
\end{equation}
for some $C_{0,\Theta}<\infty$. Iterating this inequality gives
\[
\mathbb E_{\Theta}\|\mathbf H_t\|_F^2
\leq
\alpha^t\mathbb E_{\Theta}\|\mathbf H_0\|_F^2
+
\frac{1-\alpha^t}{1-\alpha}C_{0,\Theta},
\]
which proves the claim.
\end{proof}

This result establishes moment control of the generated trajectories. It should not be interpreted
as a universal geometric-ergodicity or total-variation mixing-time guarantee, which would require
additional irreducibility and minorization assumptions. For synthetic data generation, the main
implication is that long trajectories can be generated while keeping latent magnitudes controlled.

\paragraph{Temporal memory and oscillatory responses.}
The dependence on both $\mathbf H_t$ and $\mathbf H_{t-1}$ introduces delayed feedback. This
second-order structure enables temporal autocorrelation, phase-lagged responses, and oscillatory
patterns. Around a reference state $\bar{\mathbf z}$, the conditional mean dynamics admit the local
linearization
\begin{equation}
\delta\mathbf z_{t+1}
\approx
\mathbf J(\bar{\mathbf z})\delta\mathbf z_t
+
\mathbf G(\bar{\mathbf z})\boldsymbol{\epsilon}_t,
\qquad
\mathbf z_t=
\begin{bmatrix}
\mathrm{vec}(\mathbf H_t)\\
\mathrm{vec}(\mathbf H_{t-1})
\end{bmatrix},
\end{equation}
where
\begin{equation}
\mathbf J(\bar{\mathbf z})
=
\begin{bmatrix}
a\mathbf I+D_1\boldsymbol{\mu}(\bar{\mathbf z})
&
D_2\boldsymbol{\mu}(\bar{\mathbf z})
\\
\mathbf I
&
\mathbf 0
\end{bmatrix},
\end{equation}
and $\mathbf G(\bar{\mathbf z})$ denotes the diffusion amplitude evaluated at the same reference
state. Here $D_1\boldsymbol{\mu}$ and $D_2\boldsymbol{\mu}$ are the Jacobian blocks of the vectorized
drift with respect to $\mathrm{vec}(\mathbf H_t)$ and $\mathrm{vec}(\mathbf H_{t-1})$.

\begin{proposition}[Local mechanism for quasi-periodic responses]
\label{prop:quasi_periodic}
Suppose $\mathbf J(\bar{\mathbf z})$ has a stable complex conjugate eigenvalue pair
\[
\lambda_{1,2}=re^{\pm i\theta},
\qquad
0<r<1,
\qquad
\theta\in(0,\pi).
\]
If the corresponding mode is excited by the stochastic innovation and observed by the output
projection, then the observed sequence has an enhanced resonant response near
\begin{equation}
f_{\mathrm{osc}}=\frac{\theta}{2\pi}.
\end{equation}
\end{proposition}

\begin{proof}[Proof sketch]
For the frozen-diffusion linear system, the spectral response of a linear readout is governed by
\[
(\mathbf I-e^{-i\omega}\mathbf J(\bar{\mathbf z}))^{-1}.
\]
A complex eigenvalue pair $re^{\pm i\theta}$ enhances the response near angular frequencies
$\pm\theta$. If this mode is both excited and observed, the output spectrum contains a resonant
component near $\theta/(2\pi)$.
\end{proof}

This proposition identifies a local mechanism for quasi-periodic behavior rather than exact
periodicity. With Gaussian excitation, damped oscillatory modes can be repeatedly stimulated and
appear as oscillatory autocorrelation, spectral concentration, or phase-lagged propagation.

\paragraph{Graph-spectral diversity.}
The graph filters couple temporal feedback with graph-spectral structure. If $(\mathbf q_i,s_i)$ is
an eigenpair of $\widehat{\mathbf L}$, then
\begin{equation}
\mathbf T_1\mathbf q_i=s_i\mathbf q_i,
\qquad
\mathbf T_2\mathbf q_i=(2s_i^2-1)\mathbf q_i.
\end{equation}
Thus, different graph frequencies receive different first- and second-order feedback. Since
$2s_i^2-1$ changes sign at $|s_i|=1/\sqrt{2}$, some spectral components may behave like smoothing
or diffusion modes, while others may behave like delayed-feedback or propagation modes. Although
the nonlinear transformations and channel mixing prevent a global decomposition into independent
graph modes, this view explains why diffusion-like, wave-like, and oscillatory patterns can arise in
the synthetic data.

\paragraph{Nonlinear and irregular regimes.}
The bounded nonlinearity $\tanh(\cdot)$ allows the generator to interpolate between near-linear and
saturated regimes. Small effective weights lead to mean-reverting stochastic autoregressive
behavior, whereas larger effective weights can push the drift into saturation, producing
threshold-like transitions, bursts, plateaus, and intermittent dynamics.

The deterministic skeleton can also be used to diagnose chaotic-like finite-time sensitivity.
Removing the Gaussian excitation yields
\[
\mathbf z_{t+1}=\mathbf f_{\Theta}(\mathbf z_t).
\]
Let
\[
\mathbf \Phi_T(\mathbf z_0)
=
D\mathbf f_{\Theta}(\mathbf z_{T-1})
\cdots
D\mathbf f_{\Theta}(\mathbf z_0).
\]
The finite-time Lyapunov diagnostic is
\begin{equation}
\lambda_T(\mathbf z_0)
=
\frac{1}{T}
\log
\left\|
\mathbf \Phi_T(\mathbf z_0)
\right\|_2.
\end{equation}
A positive $\lambda_T$ indicates local exponential separation of nearby initial conditions over the
finite window. Such behavior is not universal across sampled systems, but it can occur in high-gain
nonlinear regimes. Since the drift is bounded and $a\in(0,1)$, local expansion can coexist with
global moment control, yielding bounded but irregular chaotic-like trajectories.

\paragraph{Observation, masking, and corpus-level coverage.}
The observed sequence is obtained through a random projection of the latent state, so latent modes
may be amplified, attenuated, or hidden by the observation map. Random masking further introduces
external observational corruption without changing the latent transition law, mimicking realistic
settings where structured dynamics are only partially observed.

The randomization over graphs, weights, noise amplitudes, and observation maps provides corpus-level
regime diversity. If $\mathcal R$ is a set of configurations inducing a particular spatio-temporal
regime and $\nu(\mathcal R)=p_{\mathcal R}>0$, then a corpus of $M$ independently sampled systems
contains at least one system from $\mathcal R$ with probability
\begin{equation}
1-(1-p_{\mathcal R})^M.
\end{equation}
Thus, any regime with nonzero probability under the sampling distribution appears in sufficiently
large synthetic corpora with high probability.

\paragraph{Summary.}
In summary, the generator produces a controlled yet diverse family of graph-coupled stochastic
processes: (i) conditioned on each sampled system, the latent dynamics are second-order Markov with
conditionally Gaussian and heteroscedastic transitions; (ii) Chebyshev graph filtering induces
spectrally controlled spatial dependence and localized propagation; (iii) the dissipative
coefficient and bounded nonlinearities provide moment-controlled long trajectories; (iv) the
second-order recurrence enables temporal memory, delayed feedback, and quasi-periodic responses;
(v) graph-spectral filtering supports diffusion-like, wave-like, and oscillatory patterns across
different spectral components; (vi) nonlinear saturation and high-gain regimes can yield
intermittent or chaotic-like behavior; and (vii) random projection and masking induce partial and
imperfect observation. These properties align the synthetic corpus with the structural
characteristics of real-world spatio-temporal forecasting tasks while keeping the generation process
controlled and scalable.

\section{Supplementary Experimental Details}

\subsection{Evaluation Datasets}
To evaluate downstream spatio-temporal (ST) forecasting, we curate a suite of real-world benchmarks spanning traffic flow, speed, electricity consumption, and meteorology, as summarized in Table~\ref{tab:datasets}. These datasets are characterized by significant heterogeneity in spatial granularity (e.g., urban regions vs. individual sensors), temporal resolutions (ranging from 5-minute to hourly intervals), and diverse prediction targets. Such multifaceted diversity provides a rigorous testbed to validate the cross-task generalization and adaptability of FactoST across both short- and long-term forecasting horizons.

\begin{table*}[h!]
  \centering
  \caption{List of evaluation spatio-temporal datasets.}
  \setlength{\tabcolsep}{4.5pt}
  \footnotesize
  \renewcommand{\arraystretch}{1.0}
  \resizebox{\linewidth}{!}{
        \begin{tabular}{lccccc}
        \toprule
        \textbf{Dataset} & \textbf{Category} & \textbf{\# Nodes} & \textbf{Sample rate} & \textbf{Time span (Y/M/D)} & \textbf{\# Time Points} \\
        \midrule
        PEMS03 & Traffic flow & 358 & 5 minutes & 2018/09/01 -- 2018/11/30 & 26208 \\
        PEMS04 & Traffic flow & 307 & 5 minutes & 2018/01/01 -- 2018/02/28 & 16992 \\
        PEMS07 & Traffic flow & 883 & 5 minutes & 2017/05/01 -- 2017/08/31 & 28224 \\
        PEMS08 & Traffic flow & 170 & 5 minutes & 2016/07/01 -- 2016/08/31 & 17856 \\
        PEMS-BAY & Traffic speed & 325 & 5 minutes & 2017/01/01 -- 2017/06/30 & 52116 \\
        METR-LA & Traffic speed & 207 & 5 minutes & 2012/03/01 -- 2012/06/27 & 34272 \\
        \bottomrule
        \end{tabular}
    }
\label{tab:datasets}
\end{table*}

\subsection{Evaluation Metrics}
To evaluate the prediction performance, we employ two widely used regression metrics: Mean Absolute Error (MAE) and Root Mean Squared Error (RMSE). Let $\mathbf{Y} = \{Y_1, \dots, Y_N\}$ denote the ground truth of the spatio-temporal data and $\hat{\mathbf{Y}} = \{\hat{Y}_1, \dots, \hat{Y}_N\}$ represent the values predicted by the model, where $N$ is the total number of test samples. These metrics are formally defined as follows:

\begin{equation}
\text{RMSE}(\mathbf{Y}, \hat{\mathbf{Y}}) = \sqrt{\frac{1}{N} \sum_i^N \left(Y_i - \hat{Y}_i\right)^2}, \quad \text{MAE}(\mathbf{Y}, \hat{\mathbf{Y}}) = \frac{1}{N} \sum_i^N \left|Y_i - \hat{Y}_i\right|,
\end{equation}

\subsection{Baselines and Implementation}

To ensure fair comparisons, we evaluate all baseline models within a unified framework using standardized metrics (MAE and RMSE). Model hyperparameters are either retained at their default settings, as reported in the original literature, or optimized via grid search on the validation set. The specific implementation strategies for each baseline category are outlined below.

\subsubsection{Time Series Foundation Models (TSFMs)}

\begin{itemize}[leftmargin=*]
    \item \textbf{TimesFM} \cite{Das2024TimesFM}: Developed by Google Research, this is a large-scale, pretrained, decoder-only time series foundation model capable of high-accuracy univariate forecasting across various domains and frequencies. We utilize the official implementation (\url{https://github.com/google-research/timesfm}).
    The pretrained weights are sourced from \url{https://huggingface.co/google/timesfm-1.0-200m}.
    
    \item \textbf{Moirai} \cite{liu2025moirai}: A large-scale, pretrained, encoder-only time series foundation model introduced by Salesforce AI Research. 
    We adapt the official implementation from \url{https://github.com/SalesforceAIResearch/uni2ts}, utilizing the checkpoint available at \url{https://huggingface.co/Salesforce/moirai-1.0-R-base}.
\end{itemize}

\subsubsection{Spatio-Temporal Foundation Models (STFMs)}

\begin{itemize}[leftmargin=*]
    \item \textbf{UniST} \cite{yuan2024unist}: A universal STFM driven by prompt learning, pretrained across multiple urban scenarios to achieve robust generalization. As the official codebase (\url{https://github.com/tsinghua-fib-lab/UniST}) inherently supports only fixed 6-step prediction horizons, we retrain the model on 13 datasets from its original release to fit our 12- and 96-step forecasting scenarios.
    
    \item \textbf{OpenCity} \cite{li2024opencity}: A versatile STFM supporting both zero-shot and few-shot forecasting across a wide range of city-level applications. We integrate this model into our pipeline using the \texttt{Opencity-plus.pth} checkpoint.
\end{itemize}

\section{Further Discussion}

\subsection{Limitations \& Future Improvements}
\label{appendix:limitations}

While NeoST demonstrates strong zero-shot generalization and inference efficiency across diverse real-world benchmarks, we acknowledge several methodological and architectural boundaries in the current framework that provide important avenues for future research.

\textbf{Sim-to-Real Gap in Topology and Noise Distributions.} 
Our synthetic data generation engine relies on Erdős-Rényi and Poisson degree models to sample spatial graph topologies. Furthermore, the temporal evolution utilizes a standard Gaussian noise formulation, $\epsilon_t \sim \mathcal{N}(0, I)$, for stochastic excitation. However, real-world complex networks (e.g., urban traffic infrastructure, epidemiological contact networks) frequently exhibit scale-free or small-world topological properties. Similarly, exogenous shocks in physical systems often manifest as heavy-tailed distributions rather than ideal Gaussian perturbations. These idealized assumptions may introduce a ``sim-to-real'' distribution gap, potentially constraining the model's robustness against extreme anomalies or black-swan events. Future iterations of the generative engine should incorporate scale-free network generation algorithms (e.g., the Barabási-Albert model) and non-Gaussian jump processes (e.g., Lévy flights) to more accurately reflect real-world macroscopic statistics.

\textbf{Computational Scalability of Test-Time Adaptation.} 
To mitigate distribution shifts during inference, NeoST employs a Ridge Regression-based Test-Time Adaptation (TTA) mechanism. This requires computing the closed-form solution $w^{*} = (Z_{obs}^{\top}Z_{obs} + \lambda I)^{-1}Z_{obs}^{\top}\tilde{v}_{obs}$. While highly effective for the moderately sized spatial graphs evaluated in our benchmarks, the matrix inversion step incurs a cubic computational complexity. For ultra-large-scale spatial networks (e.g., nation-wide power grids or comprehensive city-level traffic systems with tens of thousands of nodes), this online adaptation step may introduce latency bottlenecks, potentially offsetting the high inference efficiency inherent to our latent-space reasoning paradigm. Future work could explore iterative solvers, gradient-based meta-learning, or low-rank approximations to accelerate this process at scale.

\textbf{Lack of Uncertainty Quantification.} 
The current NeoST architecture is designed exclusively for deterministic point forecasting, optimizing for reconstruction fidelity using the Frobenius norm in the objective function. In many critical spatio-temporal applications, such as extreme weather forecasting, epidemiological modeling, or autonomous fleet routing, point predictions are insufficient; decision-makers require reliable uncertainty quantification and confidence intervals to assess risk. The absence of native probabilistic forecasting capabilities limits the model's immediate utility in risk-sensitive scenarios. Future work will extend the Output Adapter to predict the parameters of continuous probability distributions 
or integrate quantile regression mechanisms, thereby enabling robust and calibrated uncertainty estimation for downstream tasks.

\subsection{Broader Impacts}
\label{appendix:broader impacts}

The development of NeoST, a Spatio-Temporal Foundation Model pre-trained exclusively on synthetic data, offers significant societal benefits by mitigating the demographic, geographic, and socioeconomic biases typically encoded in real-world sensor deployments. Furthermore, its efficient architecture promotes environmentally sustainable deployment compared to massive foundation models. However, deploying such models in safety-critical domains such as urban traffic management and disaster response, carries inherent risks. The "sim-to-real" distribution gap may cause the model to behave unpredictably during extreme, heavy-tailed anomalies or black-swan events, potentially endangering public safety or infrastructure security. Additionally, while our pre-training data involves no human subjects, the model's capacity to analyze sequential mobility data could be misused for mass surveillance or discriminatory practices if fine-tuned on sensitive, unanonymized tracking data. To ensure safe and equitable use, we urge practitioners to integrate robust uncertainty quantification prior to deployment, strictly adhere to human rights and data privacy regulations, and validate the model rigorously against domain-specific, real-world conditions.


\end{document}